\documentclass[letterpaper, 10 pt, journal, twoside]{ieeetran}
\IEEEoverridecommandlockouts
\usepackage{amsmath,amssymb,amsfonts}
\usepackage{algorithmic}
\usepackage{graphicx}
\usepackage{textcomp}
\usepackage{amsthm}
\usepackage{cite}
\usepackage{soul}
\usepackage{caption}
\usepackage{subcaption}

\usepackage[linesnumbered,ruled,vlined,algo2e]{algorithm2e}

\newtheorem{assumption}{Assumption}
\newtheorem{definition}{Definition}
\newtheorem{lemma}{Lemma}
\newtheorem{remark}{Remark}

\newtheorem{example1}{Example}

\newif\ifshowedits 
\showeditstrue 
\showeditsfalse

\newcommand{\blue}[1]{{\ifshowedits\color{blue}#1\else#1\fi}}
\newcommand{\red}[1]{{\ifshowedits\color{red}\st{#1}\fi}}
\newcommand{\remove}[1]{{\ifshowedits\color{red}#1\fi}}

\usepackage{xcolor}
\def\BibTeX{{\rm B\kern-.05em{\sc i\kern-.025em b}\kern-.08em
    T\kern-.1667em\lower.7ex\hbox{E}\kern-.125emX}}
\begin{document}

\title{Improving the performance of Learned Controllers \\ in Behavior Trees using Value Function Estimates \\  at Switching Boundaries 
}
\author{
    Mart Kartašev and  Petter \"Ogren 

    \thanks{The authors are with the Robotics, Perception and Learning Lab., School of Electrical Engineering and Computer Science, Royal Institute of Technology (KTH), SE-100 44 Stockholm, Sweden, {\tt\small kartasev@kth.se}}
    \thanks{Digital Object Identifier (DOI): see top of this page.}
}

\maketitle


\begin{abstract}
Behavior trees represent a modular way to create an overall controller from a set of sub-controllers solving different sub-problems. 
These sub-controllers can be created using various methods, such as classical model based control or reinforcement learning (RL). If each sub-controller satisfies the preconditions of the next sub-controller, the overall controller will achieve the overall goal. However, even if all sub-controllers are locally optimal in achieving the preconditions of the next, with respect to some performance metric such as completion time, the overall controller might still be far from optimal with respect to the same performance metric.
In this paper we show how the performance  of the overall controller can be improved if we use approximations of value functions to inform the design of a sub-controller of the needs of the next one. We also show how, under certain assumptions, this leads to a globally optimal controller when the process is executed on all sub-controllers.
Finally, this result also holds when some of the sub-controllers are already given, i.e., if we are constrained to use some existing sub-controllers the overall controller will be globally optimal given this constraint.
\end{abstract}


\section{Introduction }
\label{sec:introduction}

\IEEEPARstart{B}{ehavior trees} (BTs) are receiving increasing attention in robotics \cite{iovino2022survey,colledanchise2018behavior} where they are used to create modular reactive controllers from a set of sub-controllers solving different sub-problems. \red{ Some of these sub-controllers might have been created using RL, [3]. In this paper, we show how to improve the performance of such modular designs by iteratively computing value functions of sub-controllers, and using those value functions in the design of the sub-controllers executing before them.}\blue{In this paper, we show how to improve the performance of such modular designs when incorporating RL \cite{pereira2015framework} by iteratively computing value functions of sub-controllers and using those value functions in the design of the sub-controllers executing before them}.

BTs were originally conceived in the game AI domain \cite{islaHandlingComplexityHalo2005}, in an effort to make the controllers of in-game characters more modular, and were later shown to be optimally modular, in the sense of having so-called essential complexity equal to one \cite{biggarModularityReactiveControl2022}.

Modularity is an important property in many engineering disciplines that enables designers to solve problems by dividing them into sub-problems, that are combined into a solution for the overall problem. In the context of robot control, this might correspond to creating sub-controllers such as \emph{Move to}, \emph{Grasp}, \emph{Push} etc.
\red{When combining the sub-controllers a natural concern is to make sure each controller satisfies the preconditions of the next one.}
These controllers will then be executed in a sequence until a primary goal is reached.

However, just reaching the goal is sometimes not enough. Instead we might want to reach it in a way that is near optimal with respect to some \red{respect to some performance metric such as completion time or energy.}\blue{metric such as completion time, energy or safety}. If the overall task is associated with such a performance metric it might be that all sub-controllers are locally optimal with respect to this metric, \red{and}\blue{but} the overall controller is still far from globally optimal.

\begin{figure}
    \centering
     \begin{subfigure}[b]{0.45\columnwidth}
         \centering
         \setlength{\abovecaptionskip}{0pt}
         \includegraphics[trim={2.5cm 3cm 2cm 2cm},clip, width=4.10cm]{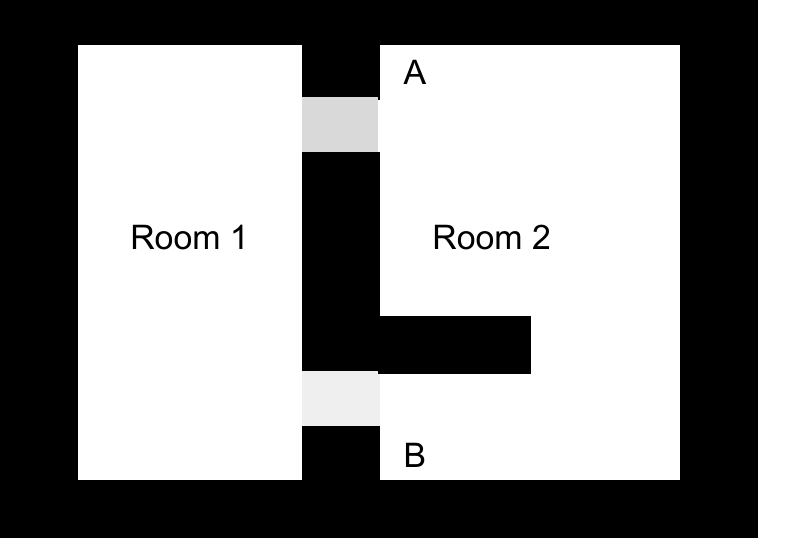}
         \caption{
         }
         \label{fig:1a}
     \end{subfigure}
    \hfill
     \begin{subfigure}[b]{0.45\columnwidth}
         \centering
          \setlength{\abovecaptionskip}{0pt}
         \includegraphics[trim={2.5cm 3cm 2cm 2cm},clip, width=4.10cm]{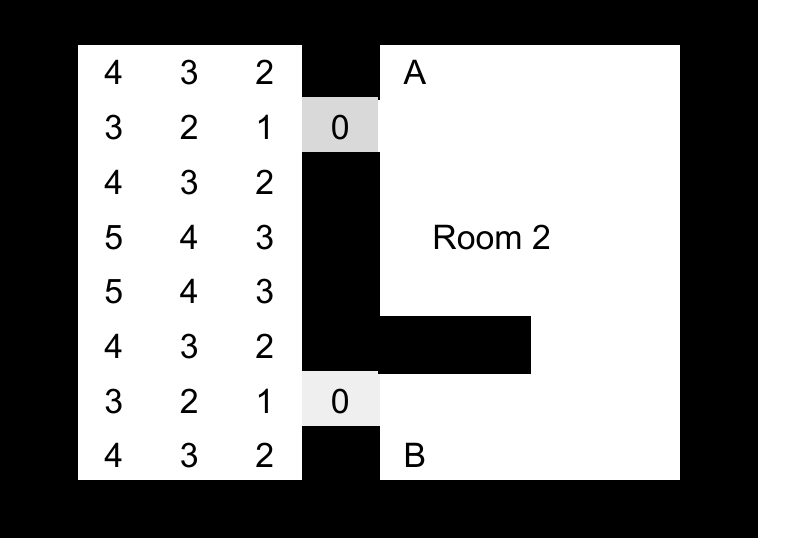}
         \caption{
         }
         \label{fig:1b}
     \end{subfigure}
     \hfill
     \begin{subfigure}[b]{0.45\columnwidth}
         \centering
          \setlength{\abovecaptionskip}{0pt}
         \includegraphics[trim={2.5cm 3cm 2cm 2cm},clip, width=4.10cm]{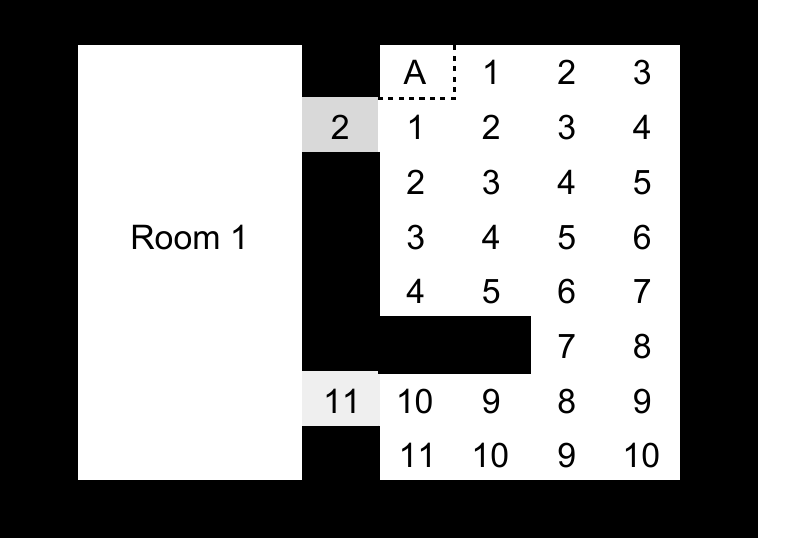}
         \caption{
         }
         \label{fig:1c}
     \end{subfigure}
 \hfill
      \begin{subfigure}[b]{0.45\columnwidth}
         \centering
          \setlength{\abovecaptionskip}{0pt}
         \includegraphics[trim={2.5cm 3cm 2cm 2cm},clip, width=4.10cm]{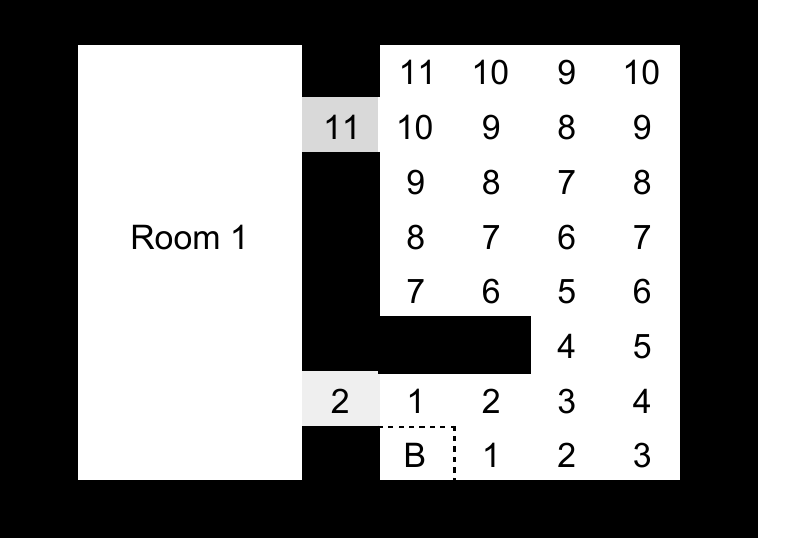}
         \caption{
         }
         \label{fig:1d}
      \end{subfigure}
      \caption{The agent first goes from room 1 to room 2, and then goes to either object A or object B, see the map in (a). The value functions for going from room 1 to room 2 can be seen in (b), the value function for going to A can be seen in (c) and the value function for going to B can be seen in (d). 
    }
 \label{fig:ex1_1}
\end{figure}

An example of this is shown in Figure~\ref{fig:ex1_1}.
Here the overall goal is to reach either position A or position B in room 2, starting from room 1. This problem is divided into three sub-controllers, \emph{Go to room 2} (\ref{fig:1b}), \emph{Go to position A} (\ref{fig:1c}) and \emph{Go to position B} (\ref{fig:1d}). The numbers in the figures show the value function in terms of the remaining time to completion (assuming one step takes one time unit). Let's assume that the controllers are to take the shortest path to their goal.
If we connect \emph{Go to room 2} with \emph{Go to position A}, the combined controller will first leave room 1 and then reach object A. However if we start in the lower half of room 1 we will exit through the lower door, from which there are 11 steps left to reach A, if we instead would have exited through the upper door, it would have taken longer time to exit room 1, but the path to A in room 2 would only be 2 steps long, leading to an overall shorter path. In this case the overall controller would not be globally optimal, even though the sub-controllers are locally optimal. 

As seen in the example above, to achieve global optimality each controller needs to be aware of how good different states are from the perspective of the next controller\red{, and}\blue{. This} information can be found using the value function (estimating the expected accumulated future reward) of the next controller \red{,evaluated} at the switching boundary. The results of doing this for \emph{Go to room 2} is shown in Figure~\ref{fig:VF_using_next}.

\blue{As an example where such issues arise in practice, consider a hypothetical agent navigating indoors with a "Move To" action. 
If completion time is penalized, a locally optimal solution will induce high speeds and accelerations.
However, this might cause problems after the switching boundary, where the robot might collide with an object, or frighten a human collaborator. With our method, the training of Move to can also take the considerations of the subsequent actions into account.
}

\red{
Looking back at the simple example, and using the value function of \emph{Go to position A} in Figure}
\remove{~\ref{fig:1c}}
\red{in the two doors between rooms 1 and 2 as reward/cost in the final step, we get a new value function of \emph{Goto room 2}, as shown in Figure}
\remove{~\ref{fig:R1gotoA}.}
\red{Doing this, the combination of the two value functions in Figures }
\remove{\ref{fig:1c}}
\red{and}
\remove{\ref{fig:R1gotoA}} 
\red{is identical to the value function one would get from treating Rooms 1 and 2 as a single room.}

\red{Similarly, if the overall task is to reach B, we can use the value function in Figure}
\remove{~\ref{fig:1d}}
\red{in the two doors between rooms 1 and 2 as reward in the final step, to get a new value function of \emph{Goto room 2}, as shown in Figure}
\remove{~\ref{fig:R1gotoB}}
\red{The core idea of this paper is to generalize this simple example to cases with several sub-controllers and higher dimensional state spaces. The sub-controllers can be created  using either reinforcement learning (RL) or some other design method. Most RL algorithms include a value function estimate and for methods that do not, we can use algorithms from the RL domain to estimate the value functions
}\remove{\cite{suttonReinforcementLearningIntroduction2018}.}

\begin{figure}
    \centering  
     \begin{subfigure}[b]{0.45\columnwidth}
         \centering
          \setlength{\abovecaptionskip}{0pt}
         \includegraphics[trim={2cm 3cm 2cm 2cm},clip, width=4.10cm]{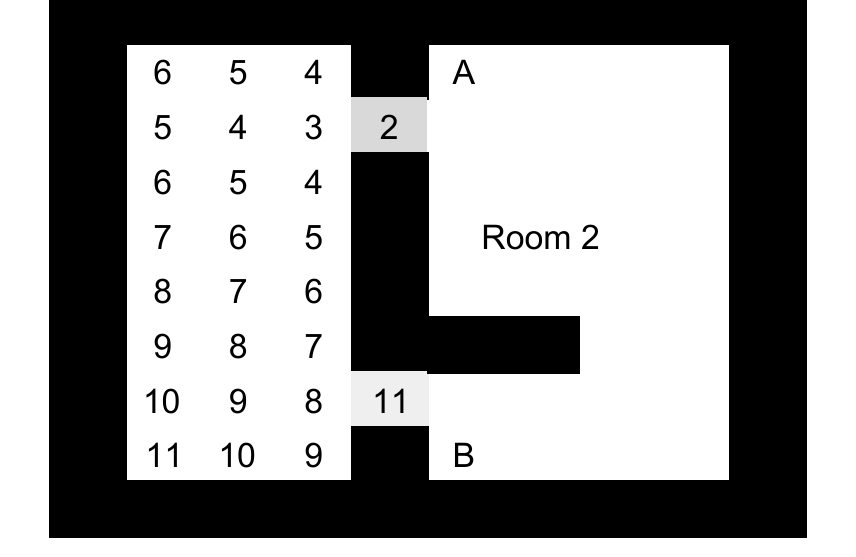}
         \caption{
         }
         \label{fig:R1gotoA}
      \end{subfigure}
\hfill
     \begin{subfigure}[b]{0.45\columnwidth}
         \centering
          \setlength{\abovecaptionskip}{0pt}
         \includegraphics[trim={2cm 3cm 2cm 2cm},clip, width=4.10cm]{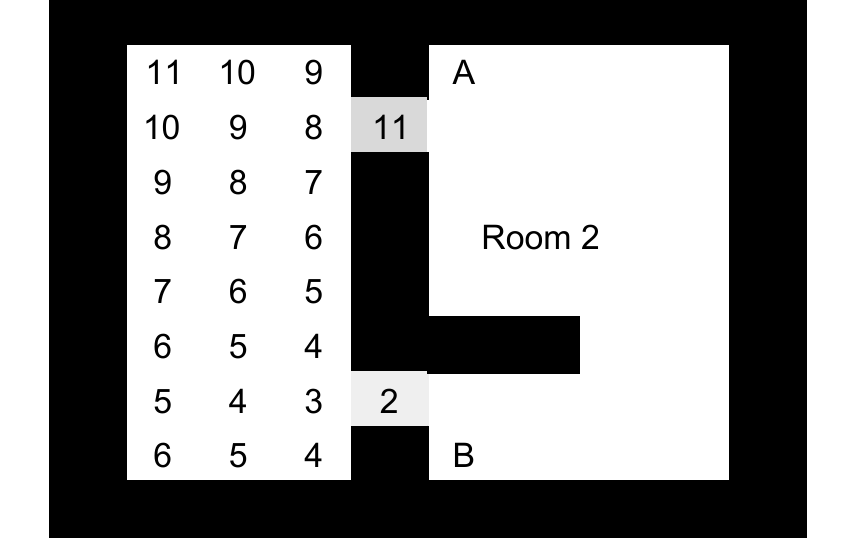}
         \caption{
         }
         \label{fig:R1gotoB}
      \end{subfigure}
    \caption{Knowing which object to go to, we can use the value function of that action, or an approximation of it, as a boundary value for the first action. Using the boundary values from object A in Figure~\ref{fig:ex1_1}(c) we get the value function of \emph{Goto room 2} shown in (a). Similarly, using the boundary values from object B in Figure~\ref{fig:ex1_1}(d) we get the value function of \emph{Goto room 2} shown in (b).
    }
    \label{fig:VF_using_next}
\end{figure}

The main contributions of this paper are as follows:

\begin{enumerate}
    \item For a set of two local RL-controllers (controllers created using RL) with given execution order, we show that if we use the value function of the second controller as a reward during the final step of the first controller, the resulting value function will satisfy the Bellman equation across the switching boundary, making the combined controller globally optimal.
    \item We extend this result to a larger set of local RL-controllers, under certain assumptions, making the combined controller globally optimal.
    \item  Given a mixed set of controllers designed using RL or with manual design principles, we extend the result above so that the global controller is optimal, under the constraint that the non-RL controllers are not changed.
\end{enumerate}

\red{ In relation to the recent publication, our approach solves a similar problem, but includes a proof of optimality, is not dependent on a user defined $\eta$-parameter, and can handle situations where information needs to be shared across more than one switching boundary, see Section 2.}

Before moving on we note two things: 
First, the proposed approach is applicable in any context where the switching boundaries between a set of policies are fixed, and some or all of those policies are implemented using RL. This includes settings such as using RL with the sequential behavior compositions in \cite{burridgeSequentialCompositionDynamically1999} and the consecutive policies of \cite{erskine2022developing}. \red{one of the most common setups in the literature that leads to such a set of switching boundaries is when BTs are used}\blue{ However, BTs are one of the most common setups that specifically lead to such static switching boundaries, 
and the computation needed to find these boundaries for an arbitrary BT can be found in \cite{ogrenBehaviorTreesRobot2022}. } Therefore, even though the results \red{are}\blue{might be} applicable outside a BT setting, we choose to present and motivate our work in the context of BTs. 

\blue{Second, note that in the example above, tailoring \emph{Goto room 2} to \emph{Goto A} makes the overall performance worse when combined with \emph{Goto B}, and vice versa. But what if you wanted to create a policy that performs well for both A and B? This can be done in two ways.
If the knowledge of what will come next can be included in the state, \emph{Goto room 2} knows whether A or B is next, and the suggested optimization will produce a policy that handles both cases optimally. 
If on the other hand the outcome of what will come next is unknown or random, we can optimize over an (estimated) distribution of subsequent policies. In the example above, if there is a 30\% probability of going to A and a 70\% probability of going to B, a weighted average of the value functions in Figures~\ref{fig:1c} and \ref{fig:1d}, ($0.3 v_A + 0.7 v_B$) can be used to optimize \emph{Goto room 2}, yielding a slight preference for the lower door, but not as strong as if it were known that B was always the target. Similarly, in the previous example of indoor navigation, we would learn that a high velocity on the boundary leads to a low reward with high probability as the value prediction decreases due to experienced collisions. Then, over time we would learn to slow down.

Thus, our approach can create a reusable policy, that is designed to optimize the expected reward over a given distribution of possible subsequent policies. This allows for modular use of actions within specific contexts, while improving expected performance, albeit with a possible trade-off in performance with respect to other contexts, that are not experienced in training or represented in the state. As seen in the examples above, our method yields the greatest performance improvement in scenarios where the locally optimal solution causes global performance degradation.
}

\red{However, a priori there is no such tradeoff. As long as a policy satisfies the same post-conditions, such as reaching room 2 in the example above, it can still be reused in other settings. Whether or not it is efficient depends on how similar the subsequent policies are.}

The remainder of this paper is organized as follows. The related work is described in Section \ref{sec:related_work}. A brief background is provided in Section \ref{sec:background}. Then the proposed approach including the theoretical results is given in Section \ref{sec:proposed_approach}, followed by a numerical example in Section \ref{sec:experimental_results}. Finally, conclusions are drawn in Section \ref{sec:conclusions}.

\section{Related Work}
\label{sec:related_work}

\red{The related work that is most similar to our approach is the recent paper}
\remove{\cite{erskine2022developing}.}
\red{There, the authors study a problem where an overall task is divided into}
\remove{$N$} 
\red{subtasks by some given transition signal}
\remove{$U(s_t)\in \{1, \ldots, N\}$.}
\red{It is not specified in}
\remove{\cite{erskine2022developing}}
\red{ how}
\remove{$U(s_t)$}
\red{is created, but it might be in the form of a BT as in our approach, making the two papers very similar. The authors of} \remove{\cite{erskine2022developing} } 
\red{also identify the problem of locally optimal tasks leading to far from optimal overall performance. Their solution is described as: "each agent is incentivized to cooperate with the next agent by training the agents policy network to produce actions that maximize both the current agents critic and the next agents critic, weighted by an introduced parameter, the cooperative ratio }
\remove{ ($\eta$)."}
\red{In their conclusion section, they point out two important items for future work: "the cooperative ratio variable }
\remove{$\eta$}
 \red{is required to be defined for each subtask" and If a specific action needs to be taken to solve the overall task and the subsequent agent is not aware of it, then CCP will still fail. Thus,}
\remove{\cite{erskine2022developing}}
\red{ is very similar to this paper in the sense that it identifies the same problem, and proposes a solution based on critics, which are estimates of the value function. However, }
\remove{\cite{erskine2022developing}}
\red{ is also different, in a number of important aspects: (a) They use a linear combination of the two critics across the entire set of states of the first controller, but we use the value function of the next task as reward only on the boundary. (b) They present no theoretical results, whereas we present conditions for convergence to the globally optimal policy. (c) They rely on a user defined parameter }
\remove{$\eta$,}
\red{  while we have no such parameter. (d) They are not able to solve problems where information is needed across more than one boundary, while our approach has no such issues. 
}

\red{The problem addressed in this paper is similar to the RL subfield known as Hierarchical Reinforcement Learning (HRL)}
\remove{\cite{pateriaHierarchicalReinforcementLearning2021}}\red{, which breaks a problem down to a set of sub-problems with respective subgoals to be reached. Each subgoal is then achieved with a separate subpolicy called a \emph{skill}.} 

RL has previously been used in BTs for training both the control switching mechanisms \cite{dey2013ql,pereira2015framework,hannaford2016simulation, fu2016reinforcement, zhang2017combining}, and the individual actions \cite{pereira2015framework} that constitute a BT. 
\blue{Many of these works handle this in the framework of Hierarchical Reinforcement Learning (HRL). 
However HRL is not applicable to our problem as the hierarchical controller we consider, with it's switching boundaries, is fixed, as given by the BT.}

In \cite{dey2013ql} the authors build a so-called Q-condition for the lowest level sequences of a BT, that can estimate the Q-value for each action in a sequence. Every tick of the BT evaluates the Q-conditions and reorders the actions in the sequences. The highest utility action of each sequence is used to determine the utility of the sequence, which can then be recursively used to reorder the entire tree, ensuring that the right action is executed at the right time.  
The concepts of \emph{learning action node} and \emph{learning composite node} are introduced in \cite{pereira2015framework}, using the Options framework \cite{suttonMDPsSemiMDPsFramework1999}  as the theoretical foundation. The learning action node encapsulates an independent RL problem with a complete definition of states, actions and rewards. In the case of the composite node, the authors use branches of the BT as actions in an RL problem to create control flow nodes. In a similar fashion, \cite{fu2016reinforcement,hannaford2016simulation,zhang2017combining} all use RL in slightly different ways to train fallbacks that optimize the choice between alternative options that achieve the same sub-goal.
Our work is different from \cite{fu2016reinforcement,hannaford2016simulation,zhang2017combining} in that we use RL to update the policies, whereas they use RL to decide when to use different policies.

As mentioned above,   
\cite{pereira2015framework} included an approach for using RL to create individual policies in a BT using a \emph{learning action node}, where a user defined MDP was created for each policy. However, these policies were solved separately, leading to the problem of locally optimal policies possibly being far from globally optimal, as discussed above.
Our work can be seen as an extension of \cite{pereira2015framework}, where, under certain assumptions, we can produce globally optimal policies by letting the reward on the switching boundary be given by the value function of the subsequent policy.

\blue{
The most similar related work can be found in 
\cite{erskine2022developing}, where the authors address the problem of local/global optimality in RL applied in different regions with fixed switching boundaries.
 However, their work is different from ours in the following aspects: (a) they provide no theoretical analysis, while we show convergence to the globally optimal policy. (b) they combine critics linearly across all states, while we use the next task's value function as reward only at the boundary. (c) they rely on a user-defined parameter ($\eta$), while our approach does not. (d) Their method cannot handle cases where preference information needs to flow across multiple switching boundaries, whereas our approach can.}

\blue{ 
Another paper that addresses a related problem is \cite{lee2018composing}. There, it is noted that switching between different controllers is sometimes problematic, as one controller might finish in a state that makes the subsequent controller fail. The proposed solution records success and failure data from experienced trajectories, and uses this to train transition policies to reach successful starting states before invoking the original controller. The approach in \cite{lee2018composing} does remove failures, but cannot optimize further than that. Our method can achieve near optimal performance, as measured by the MDP reward, whereas they focus on feasibility. We adopt the final part of a controller execution to the subsequent one, where they let the first controller finish and then execute a transition to the next skill. 
}

\section{Background}
\label{sec:background}
In this section we will give a very brief description of how BTs induce a partitioning of the statespace into different operation regions, as well as listing some key concepts and results from reinforcement learning. 

\subsection{Behaviour trees}
A BT represents a way of combining a set of sub-controllers into one overall controller in a way that is hierarchical and have been shown to be optimally modular \cite{biggarModularityReactiveControl2022}.
A recent survey can be found in \cite{iovino2022survey}, and technical overview in \cite{ogrenBehaviorTreesRobot2022}.

\red{A given BT induces a partition of the state space into a set of operating regions $\Omega_i$, where controller $i$ is executing when the state $s \in \Omega_i$.
}\remove{\cite{ogrenBehaviorTreesRobot2022}}

\blue{
 A given BT induces a partition of the state space into a hierarchical set of operating regions $\Omega_i$, where controller $i$ is executing when the state $s \in \Omega_i$ as seen in the following Lemma from \cite{ogrenBehaviorTreesRobot2022}.
\begin{lemma}
For a given node $j$, with operating region $\Omega_j$, controller $j$ is executed when $s \in \Omega_j$, and 
the operating regions of the children of $j$ is a partitioning of $\Omega_j$.
\end{lemma}

\begin{proof}
    This is a concatenation of Lemmas 7 and 8 of \cite{ogrenBehaviorTreesRobot2022}.
\end{proof}
The computation of the operating regions is a straightforward, but somewhat complex recursive operation on the tree structure, given by Definitions 10 and 11 in  \cite{ogrenBehaviorTreesRobot2022}.
}

 An illustration of this partitioning can be found in Figure~\ref{fig:7_oper_regions} below.
 In this paper we will investigate the interactions between controllers across the operating regions.

\subsection{MDPs and Reinforcement learning}

Let a Markov Decision Process (MDP) be defined as follows \cite[p232]{dietterichHierarchicalReinforcementLearning2000}:

\begin{definition}
\label{def:MDP}
(Markov Decision Process). An MDP is a 4-tuple
\begin{equation}\label{eq:markovDecisionProcess}
(S, A, p, r),
\end{equation}
where $S$ is a set of states, $A$ is a set of actions, with $A(s)\subset A$ the set of actions available at state $s$, $p(s'|s,a)=p(s_{t+1}=s'|s_t=s,a_t=a)$ is the probability of state transitions,
and
$r(s,s',a)$ is the reward for transitioning to state $s'$ from state $s$ applying action $a$. 
\end{definition}
A policy, $\pi:S \rightarrow A$, assigns an action to each state.
The value function, $v^{\pi}:S\rightarrow \mathbb{R}$, of a policy $\pi$ is the expected cumulative reward gained by the policy,
\begin{equation}\label{xx}
v^{\pi}(s)= \mathbb{E} \{ r_t + \gamma r_{t+1} + \gamma^2 r_{t+2}+ \ldots | s_t=s, \pi\},
\end{equation}
and satisfies the Bellman equation
\begin{align}
    v^{\pi}(s)&=  \sum\limits_{s'}p(s'|s,\pi(s))\  [r(s,s',\pi(s)) + \gamma v^{\pi}(s')].
\end{align}

A policy that maximizes $v^{\pi}$ is called an optimal policy $\pi^*$, and the corresponding value function, the
optimal value function $v^*$,
is the unique solution to the Bellman optimality equation \cite[p233]{dietterichHierarchicalReinforcementLearning2000}
\begin{align}
    v^*(s)&= \max_a \sum\limits_{s'}p(s'|s,a)[r(s,s',a) + \gamma v^*(s')]. \label{eq_Bellman_opt}
\end{align}

An optimal policy can be found from the optimal value function as \cite[p233]{dietterichHierarchicalReinforcementLearning2000} 
\begin{equation}
    \pi^*(s) = \text{argmax}_{a} \sum\limits_{s'}p(s'|s,a)[r(s,s',a) + \gamma v^*(s')]. \label{eq_opt_policy}
\end{equation}

\section{Proposed approach}
\label{sec:proposed_approach}
In this section we will provide the main result of the paper.
\blue{
The intended use of the results is to let the sets $\Omega_\alpha,\Omega_\beta$ below correspond to different operating regions  $\Omega_i$ of a BT.
Then, Lemma~\ref{lemma_main} shows how the policy in $\Omega_\beta$ can be designed independently of $\Omega_\alpha$, and exactly how the policy in $\Omega_\alpha$ must take $\Omega_\beta$ into account to provide overall optimality. Furthermore, Lemma~\ref{lem:comb_man_rl}
 shows how to handle the case when the policy in $\Omega_\beta$ is  already given. Finally, Lemma~\ref{lem_recursive} shows how to recursively extend these results to BTs with many operating regions.
}

\blue{
\begin{definition}[MDP-neighbors]
    Given an MDP, we say that two disjoint sets $\Omega_\alpha,\Omega_\beta \subset S$ are MDP-neighbors if there are two states $s_\alpha\in \Omega_\alpha, s_\beta \in \Omega_\beta$ and an action $a \in A$ such that the transition probability $p(s_\beta|s_\alpha, a) \neq 0$, 
    i.e. the MDP can transition from $\Omega_\alpha$ to $\Omega_\beta$ in a single step.
\end{definition}
}

\begin{assumption}
\label{ass:two_sets}
    \red{
    Assume that the two sets $\Omega_\alpha,\Omega_\beta \subset S$ are disjoint and neighboring each other in the sense that a given MDP can transition from $\Omega_\alpha$ to $\Omega_\beta$ in a single step. 
        Furthermore, assume 
    }  
    \blue{
    Assume that $\Omega_\alpha,\Omega_\beta$ are MDP-neighbors for some MDP and}
    that every trajectory of an optimal policy 
    
    1) ends in $\Omega_\beta$ with a finite accumulated reward,
    
    2) if it is in $\Omega_\beta$ it will not leave $\Omega_\beta$, 
    
    3) if it starts in $\Omega_\alpha$ it will transition to $\Omega_\beta$ without entering some other part of $S$ first.
\end{assumption}

Note that the above assumption is satisfied for many problems with a large positive reward for transition to some states inside $\Omega_\beta$ ending the episode, and a smaller negative reward for all other transitions inside $\Omega_\alpha \cup \Omega_\beta$. This might correspond to a problem where the policy should reach a goal region inside $\Omega_\beta$ using e.g., minimum time or minimum energy. The two regions might correspond to a \emph{move to} action and a \emph{pick/push object} action, where it is clear that you have to move to the object before picking/pushing it.
\blue{
However, there are also many combinations of MDPs and regions that do not satisfy Assumption 1, i.e., when the optimal policy passes the switching boundary multiple times, so it needs to be checked for each case.
}

We will use the following definition to first solve a smaller MDP in $\Omega_\beta$ and then solve another smaller MDP in $\Omega_\alpha$, using the value function of the first one as part of the reward. We call the two smaller MDPs \emph{restrictions} of the original MDP. We want them to be very similar to the larger MDP, but since an MDP cannot have transitions out of the state set, we need to add a set of states $S_{add}$ along the boundary that are absorbing (no transitions out), and give no rewards, see below.

\begin{definition}[Restriction]
    By the restriction of a MDP $P_0=(S,A,p,r)$ to $\hat S \subset S$ we mean a new MDP $\bar P=(\bar S,\bar A,\bar p,\bar r)$ with a smaller set of states $\bar S = \hat S \cup S_{add}$, where 
    $$
    S_{add}=\{s'\in S \setminus \hat S: \exists s\in \hat S,a \in A(s) \land p(s'|s,a) \neq 0 \}.
    $$ 
    The available actions are the same $\bar A(s) = A(s)$. The transition probability is given by 
\begin{equation}
\bar p(s'|s,a)= 
\begin{cases}
    p(s'|s,a),& \text{if } s \in \hat S  \label{eq_prob}\\
    1,& \text{if } s=s', s \in S_{add}\\
    0,              & \text{otherwise.}
\end{cases}
\end{equation}
Note that this makes the states in $S_{add}$ absorbing, i.e. they can never be exited.
The reward is given by 
\begin{equation}
\bar r(s,s',a)= 
\begin{cases}
    r(s,s',a),& \text{if } s, s' \in \hat S  \label{eq_reward}\\
    r(s,s',a) + v_+(s'),& \text{if } s \in \hat S,  s' \in S_{add}\\
    0,              & \text{if } s \in S_{add},
\end{cases}
\end{equation}
where $v_+(s')$ is a given function impacting the reward when transitioning from $\hat S$ to $S_{add}$.
\end{definition}
Below we will use $v_+$ to penalise undesired transitions, and reward desired transitions based on the value function in the destination state.

\begin{lemma}[Decoupled solutions]
\label{lemma_main}
Given an MDP $\mbox{P}_0=(S,A,p_0,r_0)$, with optimal value function $v_0^*(s)$, let Assumption \ref{ass:two_sets} hold for two sets $\Omega_\alpha,\Omega_\beta$.

Let the MDP $P_\beta=( S_\beta, A_\beta, p_\beta, r_\beta)$ be the restriction of 
$P_0$ to $\Omega_\beta$ with 
\begin{equation}
    v_+(s')=-\infty,  \label{eq_v_bar_beta}
\end{equation}
having the corresponding optimal value function $v_\beta^*$.

Let the MDP $P_\alpha=( S_\alpha, A_\alpha, p_\alpha, r_\alpha)$ be the restriction of 
$P_0$ to $\Omega_\alpha$ with 
\begin{equation}
v_+(s')= 
\begin{cases}
    \gamma v_\beta^*(s'),& \text{if } s'\in \Omega_\beta  \label{eq_v_bar_alpha}\\
    -\infty,& \text{otherwise } 
\end{cases}
\end{equation}
having the corresponding optimal value function $v_\alpha^*$.

Then the optimal value functions of $P_\alpha$ and $P_\beta$ are identical to the optimal value function of $P_0$ in 
$\Omega_\alpha$ and $\Omega_\beta$ respectively. That is

\begin{equation}
v_0^*(s)= 
\begin{cases}
    v_{\alpha}^*(s),& \text{if } s \in \Omega_\alpha \label{eq_lemma}\\
    v_{\beta}^*(s),& \text{if } s \in \Omega_\beta
\end{cases}
\end{equation}
\end{lemma}
Note that the lemma above states that we can solve $P_0$ by first solving the smaller MDP $P_\beta$ over $\Omega_\beta$ and then solving $P_\alpha$  over $\Omega_\alpha$ using optimal value function $v_\beta^*(s')$ on the boundary as $v_+$. Also note that once we have the optimal value function, we can easily find an optimal policy through Equation (\ref{eq_opt_policy}).

\begin{proof}

The optimal value function that solves Equation (\ref{eq_Bellman_opt}) for a given MDP is known to be unique \cite{dietterichHierarchicalReinforcementLearning2000}. Therefore we can assume that $v_0^*$ of $P_0$ is known, and if we can construct solutions $v_\alpha^*$ of $P_\alpha$ and $v_\beta^*$ of $P_\beta$  that satisfies Equations
(\ref{eq_Bellman_opt}) and  (\ref{eq_lemma}) we are done.

We start by looking at $P_\beta$. 
Let 
\begin{equation}
v_\beta^*(s)= 
\begin{cases}
    v_{0}^*(s),& \text{if } s \in \Omega_\beta \label{eq_vj}\\
    0,& \text{if } s \in S_{add}
\end{cases}
\end{equation}
Clearly this satisfies Equation $(\ref{eq_lemma})$, so it remains to show that it satisfies the Bellman optimality equation (\ref{eq_Bellman_opt}) for $P_\beta$.

We will show that the equations for both $P_0$ and $P_\beta$ do not depend on the values outside $\Omega_\beta$, and since the values inside $\Omega_\beta$ are identical, (\ref{eq_Bellman_opt}) must hold for $P_\beta$ if it holds for $P_0$.

For $P_0$ we know by Assumption~\ref{ass:two_sets} that any trajectory of an optimal policy $\pi_0^*$ of $P_0$ will remain in $\Omega_\beta$. Therefore $\pi_0^*$ must be such that $p(s'|s,\pi_0^*(s))=0$ for all $s\in \Omega_\beta, s' \not \in \Omega_\beta$. 

We also know that $\pi_0^*$ satisfies (\ref{eq_opt_policy}), therefore, for  $s\in \Omega_\beta$, the action $a$ that maximizes the right hand side of (\ref{eq_opt_policy}) is such that $p(s'|s,a)=0$ for all $s\in \Omega_\beta, s' \not \in \Omega_\beta$. 
In (\ref{eq_Bellman_opt}) the same sum is maximized over $a$, and therefore 
$p(s'|s,a)=0$ in (\ref{eq_Bellman_opt}) 
as well.

Thus, by (\ref{eq_Bellman_opt}), the values of $v_0^*(s)$ for $s\in \Omega_\beta$ do not depend on the values of $v_0^*(s)$ for $s\not \in \Omega_\beta$.

For $P_\beta$ we have that states outside $\Omega_\beta$ are the absorbing states $S_{add}$. The reward of transferring to those is $-\infty$, since $\bar r(s,s',a)=r(s,s',a)+v_+(s')=r(s,s',a)-\infty=-\infty$ by (\ref{eq_reward}).
We know that there exists a policy that avoids leaving $\Omega_\beta$ and results in a finite accumulative reward, since the available actions are the same, $\bar A(s)=A(s)$. Since  leaving $\Omega_\beta$ has a reward of $-\infty$ we conclude that the optimal policy for $P_\beta$ does not leave $\Omega_\beta$, and by the argument above, the values of $v_\beta^*(s)$ for $s\in \Omega_\beta$ do not depend on the values of $v_\beta^*(s)$ for $s\not \in \Omega_\beta$.
This concludes the proof regarding $P_\beta$.

For $P_\alpha$ we will explore how it depends on values outside $\Omega_\alpha$.
First we note that $v_\alpha^*(s)=0$ for $s\in S_{add}$ since $S_{add}$ are absorbing (no transitions out) and have reward $0$ for staying by (\ref{eq_reward}).
If $s\in \Omega_\alpha, s' \in \Omega_\beta$ we note the following 
\begin{align}
    &p_\alpha(s'|s,a)[r_\alpha(s,s',a) + \gamma v_\alpha^*(s')],  \\
    &p_\alpha(s'|s,a)[(r_0(s,s',a) +v_+(s') )+\gamma v_\alpha^*(s')] = \\
    &p_\alpha(s'|s,a)[(r_0(s,s',a) + \gamma v_\beta^*(s')) + 0] = \\
    &p_0(s'|s,a)[r_0(s,s',a) + \gamma v_0^*(s')]  \label{eq_same_rhs}
\end{align}
where we have used (\ref{eq_reward}) in the first equality,
$v^*_\alpha(s')=0$ and  (\ref{eq_v_bar_alpha}) in the second equality, and finally
(\ref{eq_prob}) and (\ref{eq_vj})    to get the last equality.

Thus we see that the right hand side of (\ref{eq_Bellman_opt}) is the same for $P_0$ and $P_\alpha$ for $s\in \Omega_\alpha, s' \in \Omega_\beta$.

For transitions inside $\Omega_\alpha$, i.e. $s,s'\in \Omega_\alpha$ the right hand sides are clearly the same. Transitions leaving $\Omega_\alpha \cup \Omega_\beta$ are handled as above as we know the optimal policy of $P_0$ never leaves $\Omega_\alpha \cup \Omega_\beta$, and the reward in $P_\alpha$ for leaving $\Omega_\alpha \cup \Omega_\beta$ is $-\infty$. Thus values outside $\Omega_\alpha \cup \Omega_\beta$ does not influence (\ref{eq_Bellman_opt}) for $s \in \Omega_\alpha$, and states $s' \in \Omega_\beta$ produce identical contributions by (\ref{eq_same_rhs}).
Therefore, if $v_0^*(s)$ satisfies (\ref{eq_Bellman_opt}) of $P_0$ so does $v_\alpha^*(s)$ of $P_\alpha$.
\end{proof}

\begin{example1}
\label{ex_ball1}
Consider a ball-shaped agent that is to push a box to a given goal region, modelled as an MDP $P_0$.
It is clear that each successful policy must first move to the box, and then push it into the goal region (as you cannot push without being close to something).
Thus we can divide $S$ into $\Omega_1$ (not close to box) and $\Omega_2$ (close to box). 

Lemma \ref{lemma_main} now tells us that we can first solve the ball pushing MDP $P_2$, and then solve the move to MDP $P_1$ using the value function from $P_2$ (which basically describes what positions around the box are beneficial for pushing it to the goal), as input.

Having done this, Lemma \ref{lemma_main} tells us that the optimal value functions of $P_0$ will be the same as $P_1, P_2$ on $\Omega_1,\Omega_2$ correspondingly. Since the optimal policy can be found from the optimal value function through (\ref{eq_opt_policy}), we have found the optimal policy to the original problem by solving the two smaller problems.

A detailed version of this example can be found in Section~\ref{sec:experimental_results} below.
\end{example1}

\begin{figure}
    \centering
    \includegraphics[width=3cm]{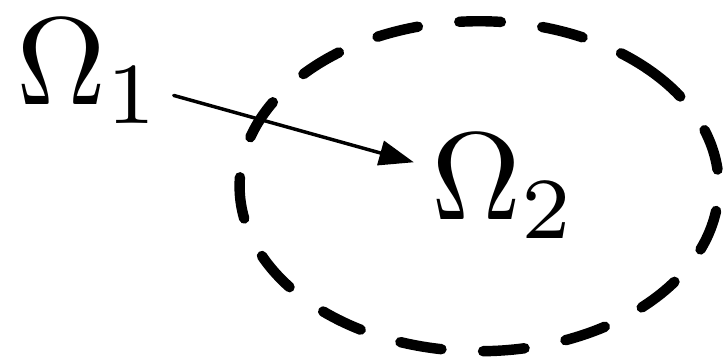}
    \caption{The basic case with $\Omega_\alpha=\Omega_1$ and $\Omega_\beta=\Omega_2$. The arrow indicates that trajectories of an optimal policy will always move from $\Omega_1$ to $\Omega_2$ and never in the opposite direction.}
    \label{fig:2_oper_regions}
\end{figure}

The next result concerns the case where we have a MDP and want to use an existing policy for parts of the solution. Examples include a classical PID controller, inverse kinematics for bringing a robot arm to some configuration, or some other policy that we want to reuse.

\begin{lemma}[Constraining the optimal policy]
\label{lemma_constraining}
Let the MDP $P_0=(S,A(\cdot),p,r)$ and a  policy $\pi_g:S \rightarrow A$ that we must execute in some domain $\Omega_g \subset S$ be given.

If we define a new function $\bar A:S \rightarrow A$ such that
\begin{equation}
\bar A(s)= 
\begin{cases}
    \{\pi_g(s)\},& \text{if } s \in \Omega_g \label{eq_a_bar}\\
    A(s),& \text{otherwise, }
\end{cases}
\end{equation}
then the optimal solution $\pi_1^*(s)$ of
the new MDP $P_1=(S,\bar A(\cdot),p,r)$
is such that
$\pi_1^*(s)=\pi_g(s)$ when $s \in \Omega_g$.
\end{lemma}
\begin{proof}
\blue{Any policy in the new MDP must satisfy $\pi(s)=\pi_g(s)$ when $s \in \Omega_g$, as this is the only available action.
Thus it also holds for the optimal policy.}

\red{    This is clear since $\pi_g(s)$ is the only available action choice in $\Omega_g$.}
\end{proof}

\begin{lemma}[Combining manual and learned sub-polices]
\label{lem:comb_man_rl}
Given a MDP $P_0=(S,A,p,r)$ and a given policy $\pi_g:S \rightarrow A$ that we want to execute in some domain $\Omega_g\subset S$.

If we constrain the available actions $A(\cdot)$ as described in Lemma \ref{lemma_constraining}
we can still apply Lemma \ref{lemma_main} \blue{if the new MDP also satisfies Assumption \ref{ass:two_sets}}.
\end{lemma}
\begin{proof}
    This is clear since applying Lemma \ref{lemma_constraining} just produces a new MDP, \blue{and Lemma \ref{lemma_main} can be applied to any MDP satisfying Assumption \ref{ass:two_sets}}.
\end{proof}

\begin{remark}
Note that this can be practical if there exists a manually designed controller, e.g., a PID- or LQR controller that we want to combine with a learned controller in an optimal way.
If $\Omega_g=\Omega_\beta$, the optimal policy in all of $\Omega_\beta$ is already known, and we can compute the optimal value function $v_\beta^*(s)$ in $\Omega_\beta$ using e.g. policy evaluation, as suggested in \cite{suttonReinforcementLearningIntroduction2018}.
\end{remark}

\begin{example1}
\label{ex_ball2}
    Looking back at the ball pushing problem in Example~\ref{ex_ball1}, we might have an existing pushing policy that we want to use. If that is the case we can constrain the available actions according to this policy, compute the value function for the given policy, and then find the optimal move to behavior using the value function of the manual push-behavior. A detailed version of this example can be found in Section \ref{sec:experimental_results} below.
\end{example1}

\begin{lemma}[Recursive application over many policies]
\label{lem_recursive}
Given a MDP $P_0=(S,A,p,r)$ where $S$ is divided into a set of disjoint operating regions $\Omega_i$, such that $S=\cup_i \Omega_i$, as illustrated in Figure~\ref{fig:7_oper_regions}.
Assume the optimal policy has a finite accumulated reward from all starting states.

Let $M\subset \mathbb{N}$ be the indices of existing policies $\pi_i: \Omega_i\rightarrow  A$ for  $i \in M$ we want to use.

First we constrain the MDP with respect to these controllers, according to Lemma~\ref{lemma_constraining}.

If the $\Omega_i$ are numbered such that transitions of optimal policies will always happen from a lower index to a higher index \blue{in this constrained MDP}, we improve the policy in any region $\Omega_i, i \not \in M$ by letting $\Omega_\alpha=\Omega_i$ and $\Omega_\beta=\cup_{j>i}\Omega_j$, and applying Lemma~\ref{lemma_main}.

If we recursively apply this strategy backwards from the highest index, we will recreate the globally optimal policy.
\end{lemma}
\begin{proof}
    Since we know that the optimal policy has a finite accumulated reward, and all transitions of the optimal policy will happen to a region with higher index $i$, items 1, 2 and 3 of Assumption \ref{ass:two_sets} is satisfied for all $\Omega_\alpha, \Omega_\beta$ constructed as above. If the assumption is satisfied, we can apply the lemma to improve performance. Since the solution in $\Omega_\alpha$ depends on the solution in $\Omega_\beta$, we will get the optimal solution by starting with the highest index and going backwards.
\end{proof}

\begin{figure}
    \centering
    \includegraphics[width=4.5cm]{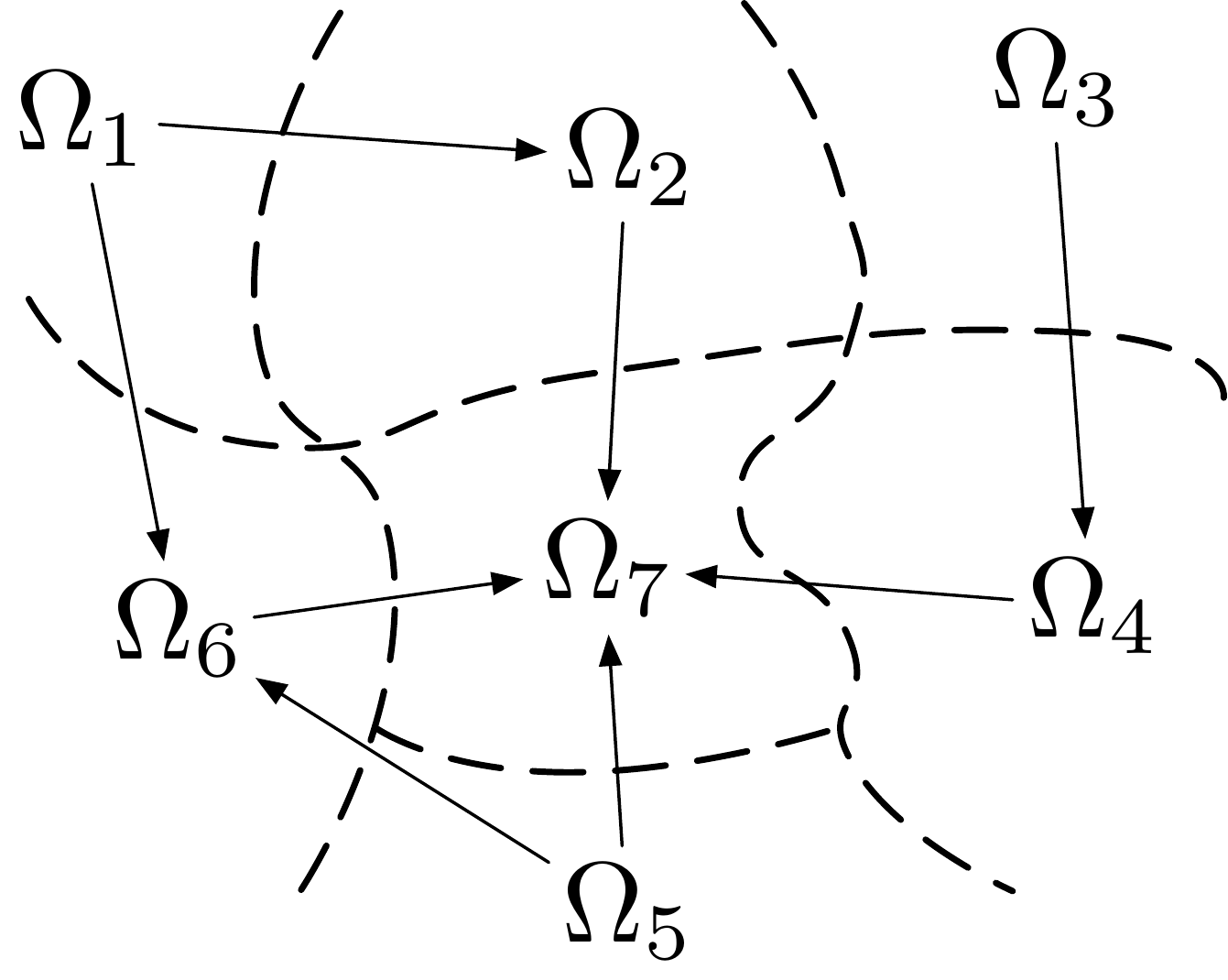}
    \caption{Illustration of the iterative application of the main result. If we know that trajectories of optimal policies will always move from $\Omega_i$ to $\Omega_j$, with $i<j$, and finally stay in $\Omega_7$, we can apply Lemma~\ref{lem_recursive} with $\Omega_\alpha=\Omega_i$ and $\Omega_\beta=\cup_{j>i}\Omega_j$, starting from the back with $\Omega_\alpha=\Omega_6$ and $\Omega_\beta=\Omega_7$, then $\Omega_\alpha=\Omega_5$ and $\Omega_\beta=\Omega_6 \cup \Omega_7$, and so on.}
    \label{fig:7_oper_regions}
\end{figure}

\section{Numerical Examples}
\label{sec:experimental_results}
A simple numerical example was already described in Section \ref{sec:introduction} and Figures \ref{fig:ex1_1} and \ref{fig:VF_using_next} above. To see how the 
 theoretical results from Section \ref{sec:proposed_approach} apply to a more dynamic example, we implemented Examples 1 and 2 above using the Unity Engine and the RL framework called ML-agents \cite{julianiUnityGeneralPlatform2020}. This section will give an overview of the setting, configurations and results. 

\subsection{Motivation of the setup}
To illustrate the theoretical results above we picked a simple example with just two policies. 
This example shows how the approach works, but not why it is needed, as the example can
easily be solved by using a standard single RL policy, as is done for performance comparison in the experiment. To motivate the approach, we note that this example is not a typical intended use case. As discussed above, BTs have been shown to be optimally modular \cite{biggar2020modularity}, and modularity is important only when you have a complex system with many policies, 
\blue{
such as the one in Figure 1 of \cite{ogrenBehaviorTreesRobot2022}. This is a typical use case, with a BT having 13 conditions to elaborately switch between 8 different actions.
However, the principles are the same, and to enable us to go into details, we pick the simple example above.
} 
\red{In such a system, the policies and switching boundaries are already there, and one can choose over what boundaries to apply the proposed approach to improve performance, as discussed above. However, it would take an extensive number of pages to present an example with many different policies solving different subtasks. Therefore, we stick to a simple problem.}

\red{We will compare the proposed approach with a single RL problem, to illustrate the main theorem of the paper, but we will not compare the approach to }
\remove{\cite{erskine2022developing}} 
\red{experimentally for two reasons. First, one would need to pick a fair }
\remove{\cite{erskine2022developing}.}  
\red{In general this is hard, but for the simple example here we can predict how their algorithm would behave for some values of $\eta$. With $\eta=0$, only using the critic of the subsequent policy, we would get the equivalent to a single RL policy, whereas for }
\remove{$\eta=1$} 
\red{we would get two completely separate ones. Second, as noted in Section 2 above and described in }
\remove{\cite{erskine2022developing}} 
\red{, their approach cannot handle problems where information is needed across two policy boundaries. Thus, to find a problem where performance differs to our advantage is straightforward. }

\subsection{Scenario setup}
As in Examples 1 and 2 above, and illustrated in Figure \ref{fig:env_example}, the environment consists of an agent in the shape of a ball, a target object in the shape of a box, and a rectangular goal area. The environment is initialized by setting a random position for the agent and the target object, as well as a randomly chosen edge for the goal area. The task of the agent is to push the target to the goal area. It can fail by exiting the mission area, or moving the target outside of the mission area. The agent's controller is split into two sub-behaviors. A \textit{Move To} behavior, which is active when the agent is far away from the box, and a \textit{Push} behavior which is activated if the agent is close to the box. Thus $\Omega_\beta=\{s \in S: ||p_{agent}-p_{box}|| < d\}$ and $\Omega_\alpha = S \setminus \Omega_\beta$, with $d$ equal to $2.5$ times the side of the box. 

\begin{figure}
    \centering
    \includegraphics[width=7cm]{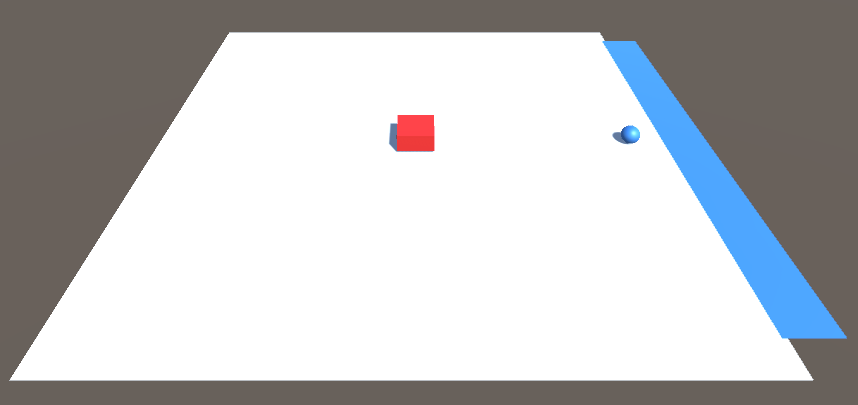}
    \caption{A snapshot of the experimental environment including a blue spherical agent, a blue goal area rectangle and a red target box. The agent fails if the target or the agent leave the plane, and succeeds if the target reaches the blue goal area.}
    \label{fig:env_example}
\end{figure}

\begin{figure}
    \centering
    \includegraphics[width=7cm]{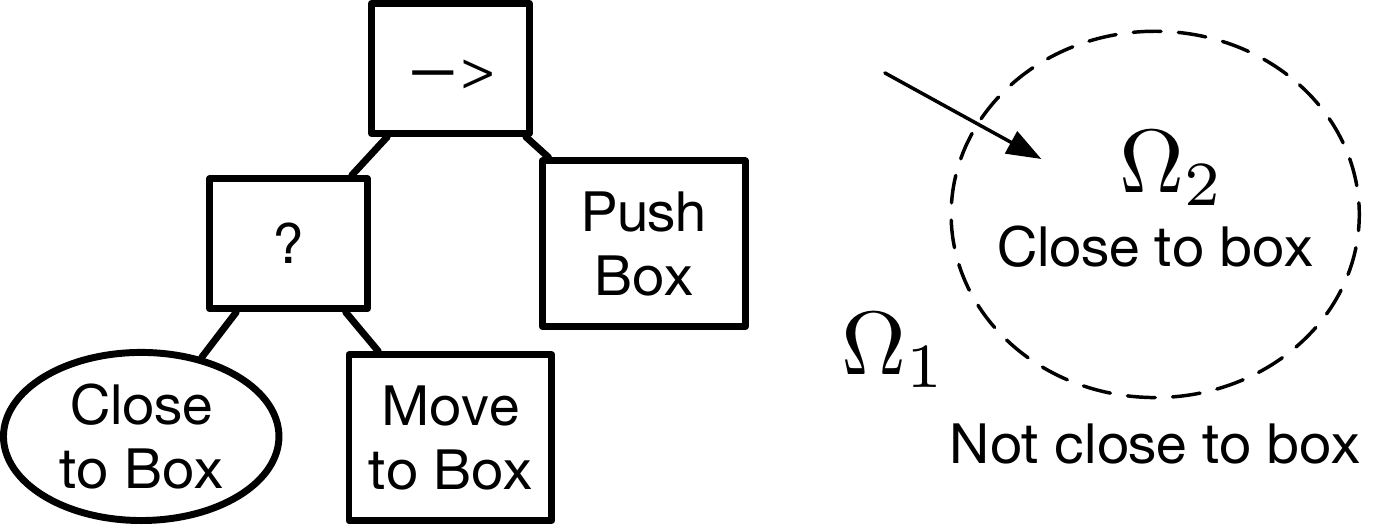}
    \caption{The simple BT of the example on the left, with one condition and two policies, and the operating regions on the right. Move to box will execute in $\Omega_1$ and Push Box will execute in $\Omega_2$.}
    \label{fig:2_oper_regions}
\end{figure}

The target and the agent are subject to second-order dynamics, including friction, as implemented by the physics engine in Unity. The agent is controlled through applying a force in the horizontal plane. The agent observes it's own position and velocity, as well as the relative position of the goal and the target box, making the state space $S$ 8-dimensional. 

\subsection{The different controllers}
First we will create five different sub-behaviors, listed on the left of Table \ref{tab:rewards}. 
\textit{Move To (Local)} will be created using RL aiming to reach the box as fast as possible. \textit{Move To (VF)} will be created using RL aiming to quickly reach a position in $\Omega_\beta$ with a good value function of the following push behavior.
\textit{Push (RL)} will be created using RL, aiming to quickly get the box to the goal area when starting in $\Omega_\beta$.  
\textit{Push (Manual)} will be created manually, aiming to do a decent job of pushing the box to the goal area when starting in $\Omega_\beta$. Finally, for comparison, we create a \textit{Single Behavior} for completing the entire task in $\Omega_\alpha \cup \Omega_\beta$ as fast as possible, using RL.
Note that the \textit{Push (Manual)} controller will be intentionally suboptimal to highlight how this effects the overall system performance. Thus \textit{Push (Manual)} first moves straight towards a position opposite the goal with respect to the target and then moves the ball straight towards the box.

\begin{table}[]
    \centering
    \begin{tabular}{l|l|c|c}
       Behavior  &  Reward every timestep & Completion & Fail\\
       \hline
       Move To (Local) &  $-0.001$ & 1  & -1\\
       Move To (VF)    &  $-0.001$ & $ \gamma v_\beta^*(s)$ & -1 \\
       Push (RL) & $-0.001 + \Delta_{dist}$ & 1 & -1\\
       Push (Manual)& $-0.001 + \Delta_{dist}$ & 1 & -1\\
       Single Behavior  & $-0.001+ \Delta_{dist}$ & 1 & -1
    \end{tabular}
    \caption{Rewards for the five different policies. \\ In the case of the manually designed push behavior, the policy is not trained, but a critic network to estimate the value function (VF) $v_\beta^*$ is. }
    \label{tab:rewards}
\end{table}

Then we create five different combinations of the sub-behaviors,
as listed on the left of Table \ref{tab:policies}.

\begin{table}[]
    \centering
    \resizebox{1\columnwidth}{!}{
    \begin{tabular}{l|c|c|c|c}
       Overall  &  Accumulated  & Success & Failure & Success   \\
       Policy  & Reward &  Duration (steps)& Duration (steps) & Rate (\%) \\
       \hline
       $\pi_A$, Move To(Local) + Push (RL)  & $1.39 \pm 0.48$ & $248.02 \pm 90.97$  & $171.12 \pm 58.85$ & $95.20$ \\
       $\pi_B$, Move To(Local) + Push (Manual) & $-0.07 \pm 1.20$ & $237.88 \pm 120.40$  & $148.31 \pm 57.84$ & $39.59$ \\
       $\pi_C$, Move To(VF) + Push (RL)     & $\mathbf{1.57} \pm 0.18$ & $204.01 \pm 56.21$  & $175.65 \pm 45.15$ & $\mathbf{99.86}$ \\
       $\pi_D$, Move To(VF) + Push (Manual) & $1.40 \pm 0.51$ & $275.67 \pm 138.14$  & $280.56 \pm 153.34$ & $96.69$ \\
       $\pi_E$, Single Behavior  & $1.52 \pm 0.17$ & $\mathbf{192.49} \pm 56.41$  & $128.50 \pm 80.01$ & $98.76$
    \end{tabular}
    }
    \caption{Evaluated policies. Mean and standard deviation across 10 000 random episodes post training. Best results indicated by bold numbers. }
    \label{tab:policies}
\end{table}
Note that the two \textit{Move To (VF)} are different, as they are trained with different value functions on the boundary, coming from \textit{Push RL} and \textit{Push Manual} respectively.

\subsection{RL formulation}

The rewards of the RL are summarized in Table \ref{tab:rewards}.
All sub-behaviors receive a small negative reward of $-0.001$ for each passing time step and a small reward $\Delta_{dist}$ for moving the box closer to the goal area, $\Delta_{dist}=(d_{goal}(s_{t-1}) - d_{goal}(s_{t}))/d_{start}$, with $d_{goal}$ being the distance from the box to the goal area, and $d_{start}$ this distance at the start, for normalization.
At the end of the episode, all sub-behaviors
receive a negative reward of $-1$ if the agent or the target leaves the mission area, and a positive reward of $1$ or $\gamma v_\beta^*(s)$ for completion of the task. Above, $v_\beta^*(s)$ is the value function of the following sub-behavior, i.e., either the manual or RL-version of Push.
All actions and observations are normalised to a range of $[-1,1]$. 

The training is done using the built-in implementation of PPO \cite{schulmanProximalPolicyOptimization2017} from the Unity ML-Agents package \cite{unity}. In the case of the manual behavior \textit{Push (Manual)}, we only train a critic network to estimate the value function, using the normal PPO loss function \cite{schulmanProximalPolicyOptimization2017}, with the policy loss removed.

\red{The RL was performed by training each of the sub-behaviors whenever they are active in the environment. That is, we don't use individual training environments for each behavior. Instead, the agent takes actions according to the currently active behaviour's policy. All experiences gathered during this time will be used only for training the currently active behavior. Whenever the agent crosses the switching boundary, the active episode ends and a new one is started for the policy that was switched to.}
\blue{The RL was executed by training each sub-behavior when active in a single environment, as opposed to using separate training environments for each behavior. The agent operates according to the policy of the currently active behavior, using the gathered experiences for training that specific behavior. Upon crossing a switching boundary, the ongoing episode ends, and a new one is started for the switched-to policy.}

As suggested in Lemma \ref{lemma_main}, we first train \textit{Move To (Local)} and \textit{Push (RL)}, and estimate the value function of \textit{Push (Manual)}. This creates the components of $\pi_A$ and $\pi_B$. Then we train the two versions of \textit{Move To (VF)}, using the value functions from \textit{Push (RL)} and \textit{Push (Manual)} respectively, creating the components of $\pi_C$ and $\pi_D$.
Finally, we train \textit{Single Behavior}, for $\pi_E$.

\subsection{Results}

\blue{We present the data in three parts. First, we discuss post-training evaluation results. Providing context to these results, we then analyze reward graphs and histograms of the training data, comparing outcomes with and without utilizing the value function as a reward signal. Finally, we measure the disparity between the value function of the single learned policy and those of the split models to illustrate the theory in Lemma \ref{lemma_main}.}

\red{In order to present the results, we split the data into three parts.
We start by discussing the post-training evaluation results. In order to give those results some context, we then look at the training data by analysing the reward graphs and histograms, comparing results with and without the use of the value function as a reward signal. Additionally, we will measure the difference between the value function of the single learned policy to those of the split models, in order to illustrate the theory in Lemma 2.} 

\red{An overall evaluation of running the five different policies from 10000 random starting states can be found in Table II. As predicted by theory, $\pi_C$ and $\pi_E$ are the best ones. Given additional training time, the differences should decrease, converging to the same optimal policy. For now we see that there is a small difference, with $\pi_C$ producing slightly higher rewards and less failures, and $\pi_E$ completing slightly faster, when not failing. Both of $\pi_C$ and $\pi_E$ perform better than the locally optimal $\pi_A$. Comparing the two versions with the manual push behavior, $\pi_B$ and $\pi_D$, we see a huge difference in success rate, with $\pi_B$ failing about $60\%$ of the runs. Despite the sub-optimal manual behavior, the proposed method allows $\pi_D$ to produce reasonable handover states, resulting in a success rate and accumulated reward comparable to the RL policies of $\pi_A$, at the price of slightly higher completion times.}

\blue{Table \ref{tab:policies} presents an overall evaluation of running five policies from 10 000 random starting states. As theory suggests, $\pi_C$ and $\pi_E$ perform best, with slight differences. $\pi_C$ yields marginally higher rewards and fewer failures, and $\pi_E$ completing tasks a bit faster when not failing. Both $\pi_C$ and $\pi_E$ outperform the locally optimal $\pi_A$. Comparing versions with manual push behavior, $\pi_B$ has a significant 60\% failure rate, while despite the sub-optimal manual behavior, $\pi_D$ achieves success rates and accumulated rewards comparable to RL policies like $\pi_A$, albeit with slightly longer completion times.}

The historical training reward, per training session, can be seen in Figures \ref{fig:learn_reward} and \ref{fig:manual_reward}. Both graphs feature the single model behavior as a baseline for comparison. We can see that after an initial decrease in reward, all training configurations have a stable upwards trend that reaches an equilibrium, with the reward plateauing around a specific value.

\begin{figure}[]
    \centering
\includegraphics[width=0.85\columnwidth]{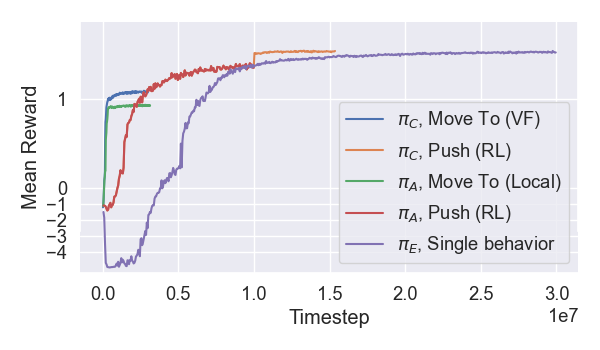}
    \caption{Training results for $\pi_A$, $\pi_C$ and $\pi_E$, using Push (RL).}
    \label{fig:learn_reward}
\end{figure}

\begin{figure}[]
    \centering
    \includegraphics[width=0.85\linewidth]{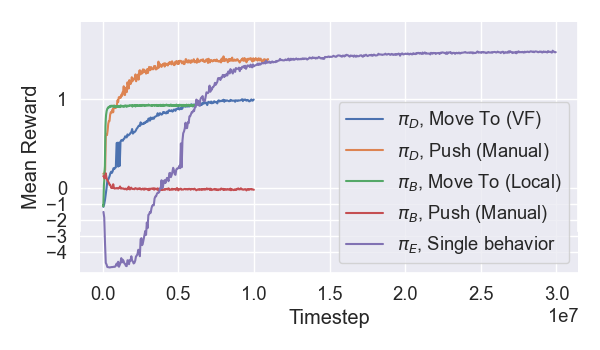} 
    \caption{Training results for $\pi_B$, $\pi_D$ and $\pi_E$, using Push (Manual).}
    \label{fig:manual_reward}
\end{figure}

In Figure \ref{fig:learn_reward} we see the training results of $\pi_A$, $\pi_C$ and $\pi_E$.
$\pi_E$ is trained first to use as a reference.
Then we train $\pi_A$ that does not use the value function to connect the two sub-behaviors. As can be seen, the mean reward converges quickly for 
\textit{Move To (Local)}(green) and \textit{Push (RL)}(red).
Then we train $\pi_C$, by training \textit{Move To (VF)}(blue) using $v_\beta^*(s)$ from \textit{Push (RL)}. As can be seen, this achieves a higher mean reward than \textit{Move To (Local)}. This is reasonable, as the final reward given by $v_\beta^*(s)$ is sometimes larger than one.

In Figure \ref{fig:manual_reward} we see the training results of $\pi_B$, $\pi_D$ and $\pi_E$.
When we train $\pi_B$ an interesting thing happens. As \textit{Move To (Local)} gets better, the performance of the fixed manual controller
\textit{Push (Manual)}(red), decreases. This is due to the fact that the manual controller is not very robust, and as \textit{Move To (Local)} learns to enter $\Omega_\beta$ in close to minimum time, using high velocities, the starting states handed to \textit{Push (Manual)} become more difficult.

When training $\pi_D$ we used the value function $v_\beta^*(s)$ from \textit{Push (Manual)} as final reward. In this way, \textit{Move To (VF)} is made aware of the capabilities of the following \textit{Push (Manual)}. The average reward of \textit{Move To (VF)} rises slower than \textit{Move To (Local)}, but at the same time the average reward of the static \textit{Push (Manual)} continues to increase.
In the end, the combined result of the constrained $\pi_D$ is slower, but almost as reliable as the near optimal unconstrained $\pi_E$, see Table \ref{tab:policies}.

\begin{figure}[]
    \centering
    \includegraphics[width=0.95\columnwidth]{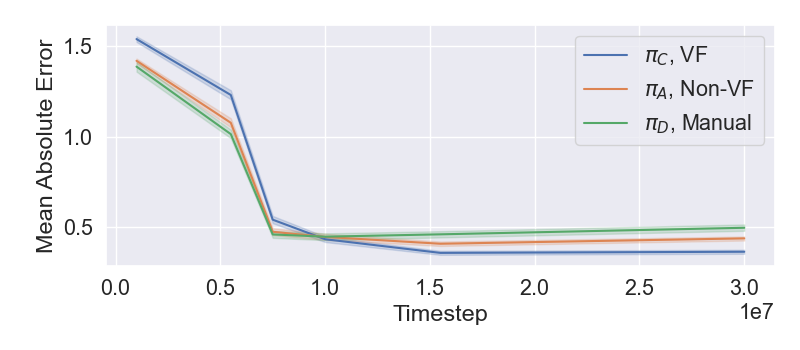} 
    \caption{Mean absolute difference between the value functions of the experimental configuration and single learned model value function with 95\% confidence interval. Computed over 100 example states per model iteration, 625 samples per state.}
    \label{fig:vf_error_mae}
\end{figure}

\subsection{Comparing the Value Functions}

Lemma \ref{lemma_main} shows that the optimal value functions of the original MDP $P_0$ is identical to the optimal value functions of the two restricted MDPs $P_\alpha$ and $P_\beta$.

Running an RL algorithm we know that as time tends to infinity, the learned value function will converge towards the optimal one \cite{suttonReinforcementLearningIntroduction2018}.
We measured the difference  between the value function of $\pi_E$ and the two restricted ones for the three policies $\pi_A$, $\pi_C$ and $\pi_D$.
The results, seen in Figure \ref{fig:vf_error_mae}, show that the difference is smallest for $\pi_C$, as predicted by the Lemma.
Note also that after the initial drop, the difference to $\pi_A$ and $\pi_D$ actually increases as training progresses, which is reasonable as they are not expected to converge to the same optimal value function.

The difference in value functions is more distinct in the Manual case, since the \textit{Push} behavior is fixed causing the \textit{MoveTo} to try to compensate. This means that the overall behavior will diverge more strongly from the optimal single behavior case, with both sub-behavior policies shaped differently. With $\pi_A$, the difference would primarily be in the \textit{MoveTo} behavior, causing the error the be smaller.

\section{Conclusions}
\label{sec:conclusions}
We investigated solutions to the local optimality issue that occurs when trying to combine learned controllers  into a handcrafted policy. Leveraging the BT structure's deterministic switching between policies, we used the value function estimate of the controller being switched to as a final reward of the previous controller. 
The proposed approach goes beyond the state of the art in that we provide theoretical guarantees of optimality, are not using an arbitrary design parameter, and can handle problems where information has to flow across more than one switching boundary.

\bibliographystyle{IEEEtran}
\bibliography{references,MyLibraryPetter}
\end{document}